\documentclass{article} 
 \usepackage[preprint,nonatbib]{neurips_2023}
\usepackage[numbers]{natbib}
\usepackage[utf8]{inputenc} 
\usepackage[T1]{fontenc}    
\usepackage[svgnames]{xcolor}
\usepackage{hyperref}
\hypersetup{
     colorlinks=true,
     linkcolor=black,
     filecolor=blue,
     citecolor = orange,      
     urlcolor=cyan,
     }
\usepackage{url}            
\usepackage{booktabs}       
\usepackage{amsfonts}       
\usepackage{nicefrac}       
\usepackage{microtype}      
\usepackage{amsmath}
\usepackage{graphicx} 

\usepackage{graphicx}
\usepackage{subfigure}
\newcommand{\bw}{\boldsymbol{w}}
\newcommand{\bu}{\boldsymbol{u}}
\newcommand{\bom}{\boldsymbol{m}}
\newcommand{\ef}{\operatorname{ef}}
\newcommand{\triangleq}{\overset{\triangle}{=}}

\newcommand{\ie}{\textit{i}.\textit{e}.~}

\usepackage{algorithmicx,algorithm}

\definecolor{lightpink}{rgb}{1.0, 0.71, 0.76}  
\definecolor{lightgreen}{rgb}{0.56, 0.93, 0.56}  

\usepackage{amsmath,amsfonts,bm}
\usepackage{amsthm}

\def\eqref#1{equation~\ref{#1}}

\def\1{\bm{1}}

\DeclareMathAlphabet{\mathsfit}{\encodingdefault}{\sfdefault}{m}{sl}
\SetMathAlphabet{\mathsfit}{bold}{\encodingdefault}{\sfdefault}{bx}{n}

\def\bv{\boldsymbol{v}}
\def\bA{\boldsymbol{A}}

\newcommand{\GD}{\operatorname{GD}}
\newcommand{\GDM}{\operatorname{GDM}}

\newtheorem{proposition}{Proposition}

\usepackage{hyperref}
\usepackage{url}

\title{When and Why Momentum Accelerates SGD:\\ An Empirical Study}

\author{Jingwen Fu$^{1}$\footnotemark[1], Bohan Wang$^{2}$\footnotemark[1], Huishuai Zhang$^{3}$\footnotemark[2], Zhizheng Zhang$^{3}$, Wei Chen$^{4}$\footnotemark[2], Nanning Zheng$^{1}$\footnotemark[2] \\  fu1371252069@stu.xjtu.edu.cn            \\ 
bhwangfy@gmail.com\\ \{huishuai.zhang, zhizzhang\}@microsoft.com\\ chenwei2022@ict.ac.cn\\
        nnzheng@mail.xjtu.edu.cn \\
       $^{1}$Xi'an Jiaotong University, $^{2}$University of Science and Technology of China \\   
      $^{3}$Microsoft Research, $^{4}$Institute of Computing Technology, Chinese Academy of Sciences
}

\def\month{MM}  
\def\year{YYYY} 
\def\openreview{\url{https://openreview.net/forum?id=XXXX}} 

\begin{document}

\maketitle

\renewcommand{\thefootnote}
{\fnsymbol{footnote}}
\footnotetext[1]{Equal contribution. Work done during internships at Microsoft Research Asia.}
\footnotetext[2]{Corresponding Authors}
\begin{abstract}

Momentum has become a crucial component in deep learning optimizers, necessitating a comprehensive understanding of when and why it accelerates stochastic gradient descent (SGD). To address the question of ''when'', we establish a meaningful comparison framework that examines the performance of SGD with Momentum (SGDM) under the \emph{effective learning rates} $\eta_{\ef}$, a notion unifying the influence of momentum coefficient $\mu$ and batch size $b$ over learning rate $\eta$. In the comparison of SGDM and SGD with the same effective learning rate and the same batch size, we observe a consistent pattern: when $\eta_{\ef}$ is small, SGDM and SGD experience almost the same  empirical training losses; when $\eta_{\ef}$ surpasses a certain threshold, SGDM begins to perform better than SGD.  Furthermore, we observe that the advantage of SGDM over SGD  becomes more pronounced with a larger batch size. 
For the question of ``why'', we find that the momentum acceleration is closely related to \emph{abrupt sharpening} which is to describe a sudden jump of the directional Hessian along the update direction. Specifically, the misalignment between  SGD and SGDM happens at the same moment that SGD experiences abrupt sharpening  and converges slower. Momentum improves the performance of SGDM by preventing or deferring the occurrence of abrupt sharpening. Together, this study unveils the interplay between momentum, learning rates, and batch sizes, thus improving our understanding of momentum acceleration.

\end{abstract}

\begin{figure}[h]
    \centering
    \includegraphics[width=0.95 \textwidth]{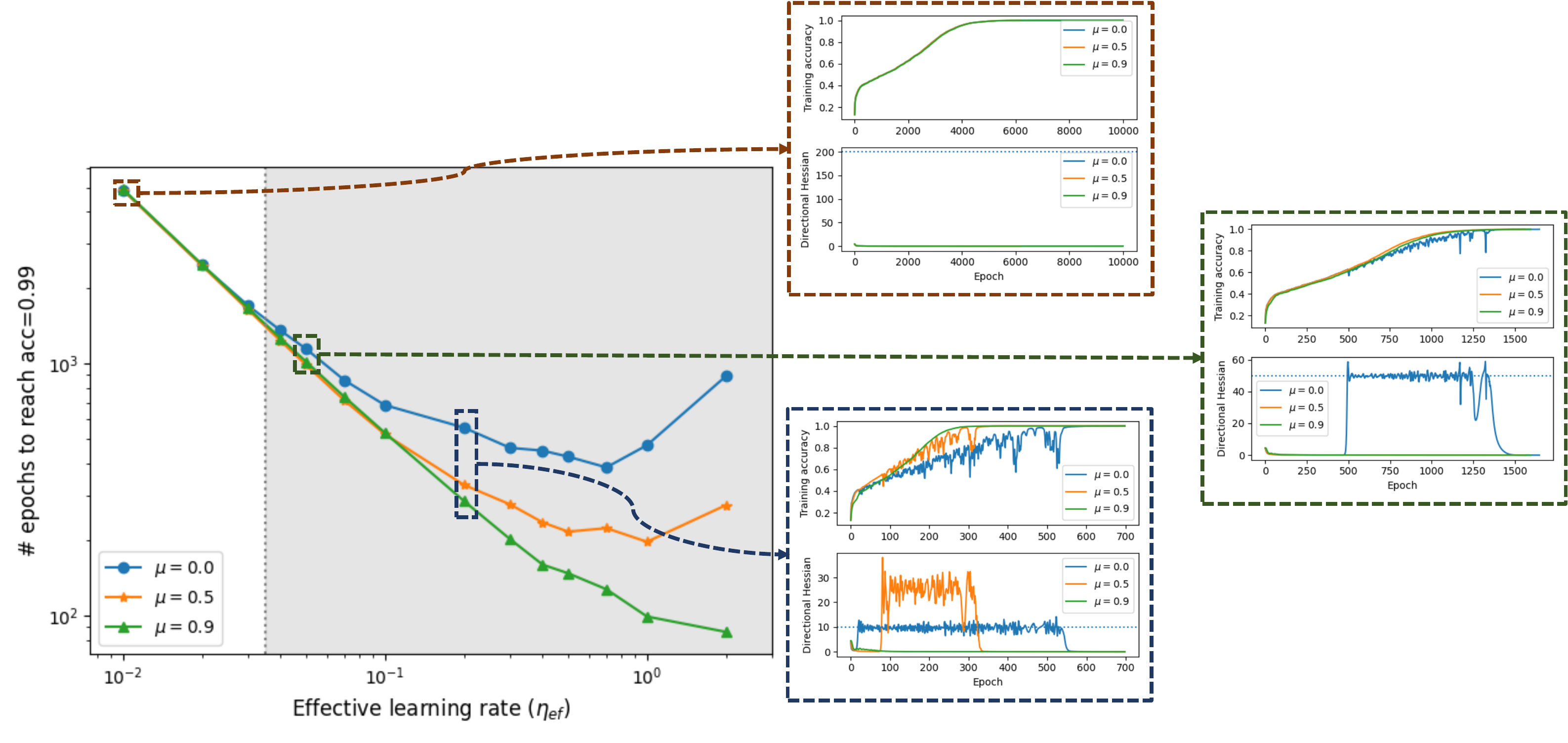}
    \caption{\textbf{The slow-down of a optimizer occurs at the same moment when it experiences abrupt sharpening.} 1) Before experiencing any abrupt sharpening, the training speeds of all optimizers are aligned. The directional Hessian along update direction for each optimizer remains small throughout the training process. 2) The convergence of a optimizer slows down after experiencing a sudden jump of the directional Hessian along update direction, called ``abrupt sharpening''. 3) Momentum defers the abrupt sharpening, thereby helps to accelerate.}
    \label{fig:overview}
\end{figure}
\section{Introduction}
One key challenge in deep learning is to effectively minimize the empirical risk $f(\bw)=\frac{1}{N}\sum_{i=1}^N \ell(\bw,z_i)$,  where $\ell$ the loss function, $\{z_i\}_{i=1}^N$ is the dataset, and $\bw$ is the parameter of deep neural networks. To tackle this challenge, countless optimization tricks have been proposed to accelerate the minimization, including momentum \cite{polyak1964some}, adaptive learning rate \cite{kingma2014adam}, warm-up \cite{goyal2017accurate}, etc.  Among these techniques, momentum, which accumulates gradients along the training trajectory to calculate the update direction, is undoubtedly one of the most popular tricks. Momentum has been widely adopted by state-of-art optimizers including Adam \cite{kingma2014adam}, AMSGrad \cite{reddi2019convergence}, and Lion \cite{chen2023symbolic}.

The widespread use of momentum  necessitates the understanding of \textbf{when} and \textbf{why} momentum works, which can either facilitate a good application of momentum in practice, or help building the next generation of optimizers. However, it is more than surprising that neither of when and why momentum works in deep learning is clear, even for the simplest momentum-based optimizer stochastic gradient descent with momentum (SGDM). 
As for \textbf{when},  \citet{kidambi2018insufficiency} argue that the benefit of momentum appears when the batch size is small, while \citet{shallue2019measuring} find that SGDM outperforms SGD when the batch size is large (close to full batch).  For a fixed batch size, \citet{leclerc2020two} even observe that SGDM can perform either better or worse than SGD depending on the learning rate. These diverse findings make the effect of momentum rather complicated and confusing. As for \textbf{why}, the theoretical benefit of momentum is only justified for convex objectives \cite{polyak1964some}. For highly non-convex deep neural networks, \citet{cutkosky2020momentum} attributes the benefit of momentum to its ability  of canceling out the noise in stochastic gradients, which however, contradicts to the observation that the benefit of momentum is more pronounced for larger batch sizes where the noise is smaller \cite{kunstnernoise}. 

In this paper, we investigate the underlying acceleration mechanism of momentum by systematically comparing the performance of SGD and SGDM. The psedo-code of SGD and SGDM \footnote{We focus on Polyak's momentum in the main text, and verify that our conclusions also hold for Nesterov's momentum in the Appendix \ref{sec:Nesterov}.} is given in Algorithm \ref{alg:sgd}, where SGD is obtained by ignoring the green-highlighted content. 

To proceed, we first establish a meaningful comparison framework to avoid unnecessary complexities. Specifically we compare SGD and SGDM or two SGDMs with different momentum coefficients under a same \emph{effective learning rate}.  The effective learning rate is the learning rate multiplied by a factor $1/(1-\mu)$ for a momentum coefficient $\mu$, which takes into account  the first-order approximation of the update magnitude induced by momentum. Moreover, we evaluate their performances for a wide range of effective learning rates which cover typical choices. 

Our comparison framework gives us a full picture of their performances, which clearly reveals \textbf{when} momentum helps acceleration.  Based on this framework, we observe a consistent pattern: the training speeds of SGD and SGDM are almost the same for small effective learning rates, but when the effective learning rate increases beyond a certain threshold, SGDM begins to perform better than SGD and shows acceleration benefits. 

When varying batch sizes,  the effective learning rate is further multiplied by a factor $k/b$ for a batch size $b$ with respect to a reference batch size $k$, which takes into account the linear scaling effect of batch sizes over the learning rate. We find that reducing batch size has a similar effect as adding momentum: SGDMs with different batch sizes perform similarly with small effective learning rate, while SGDM with smaller batch size outperform that with larger batch size when effective learning rate is larger than a certain threshold. 

To understand why momentum accelerates training, we first focus on the full batch case and study GD(M) since the acceleration effect of momentum is more evident with large batch size. We find that during the training process, the loss of GD deviates from that of GDM at the same time when GD starts to oscillate. We further attribute the oscillation to a phenomenon of ``abrupt sharpening'' that the directional Hessian along the update direction first stays around $0$ and then experiences a sudden jump which leads to oscillation. We show abrupt sharpening is a new feature of the renowned concept Edge of Stability, and more importantly, it can be used to theoretically explain the alignment and deviation between GD and GDM: before GD and GDM exhibit abrupt sharpening, the gradient barely change and the updates of GD and GDM are close; abrupt sharpening slows down the convergence; and momentum can defer abrupt sharpening and thus accelerate. We demonstrate this methodology through Figure \ref{fig:overview}.

We further find that a smaller batch size can also defer abrupt sharpening, which overlaps with the effect of momentum. This coincides with the observation that the benefit of momentum is more pronounced for the case with large batch sizes.

In summary, we empirically investigate the benefit of momentum  and our contributions are as follows.

\begin{itemize}
\item We introduce a meaningful  framework to compare the performances of SGD and SGDM with effective learning rates, which gives a full picture of the momentum benefits.

\item \textbf{When?} The comparison clearly reveals that momentum helps acceleration when the effective learning rates and the batch sizes are large.
\item \textbf{Why?}     We show that once the optimizer experiences abrupt sharpening, the training process slows down and the momentum can significantly postpone the point of abrupt sharpening.
\end{itemize}

\begin{algorithm}
    \centering
    \caption{SGD   and \colorbox{lightgreen}{SGDM}}\label{alg:sgd}
    \hspace*{0.02in} 
    
    \begin{algorithmic}[1]
        \State \textbf{Input:} the loss function $\ell(w,z)$, the initial point $\bw_{1} \in \mathbb{R}^d$,  the batch size $b$, learning rates $\{\eta_t\}_{t=1}^{T}$, \colorbox{lightgreen}{  $\bom_0=0$, and momentum hyperparameters $\{\mu_t\}_{t=1}^{T}$}.
        \State \textbf{For} $t=1\rightarrow T$:
        \State ~~~~~Sample a  mini-batch of data $B_t$ with size $b$ 
       \State ~~~~~Calculate stochastic gradient $\nabla f_{B_t}(w_t)=\frac{1}{b}\sum_{z\in B_t}\ell (w_t,z)$
       \State ~~~~~Update $\bom_{t}\leftarrow$\colorbox{lightgreen}{$\mu_t \bom_{t-1}+$}$\nabla  f_{B_t}(\bw_{t})$
        \State ~~~~~Update $\bw_{t+1}\leftarrow\bw_t-\eta_t \bom_t$
        \State \textbf{End For}
    \end{algorithmic}
\end{algorithm}

\section{Related Works}

\paragraph{Effect of momentum in optimizers.} In Polyak's original paper \cite{polyak1964some}, there is a theoretical proof that GDM converges faster than GD for strongly convex objective functions, and such an analysis is extended to SGDM latter \cite{bollapragada2022fast}. However, no concrete theoretical explanation exists for the effect of momentum over non-convex objectives, including deep neural networks, with all of the theoretical analysis of SGDM over non-convex objective functions provides no faster convergence rate than SGD \cite{liu2020improved,defazio2020momentum}. The benefit of SGDM in deep learning tasks is attributed to its ability to cancel out the noise in stochastic gradient by some works  \cite{cutkosky2020momentum}, but this contradicts to the latter
 experiments about the SGD and SGDM in \cite{kunstnernoise,kidambi2018insufficiency,shallue2019measuring} varying the batch size, showing that noise may not explain the effect of momentum.  \citet{leclerc2020two} observe that whether SGDM performs better further ties to the learning rate, and SGDM can be worse than SGD when the learning rate is large.

\paragraph{Edge of Stability.} \cite{cohen2021gradient} discovers a negative correlation between the sharpness of objective functions in the training process of deep learning tasks and the learning rate, called "Edge of Stability" (EoS). Specifically, when using gradient descent (GD) with learning rate $\eta$, they observe that the sharpness will first progressively increase, and then hover at $\frac{2}{\eta}$. Similar phenomena are latter observed in other optimizers including SGD, SGDM, and Adam \cite{cohen2022adaptive}. Traditionally, optimization analysis requires sharpness  to be smaller than $\frac{2}{\eta}$ to ensure convergence. This is, however, violated by EoS, and several works have tried to understand such a mismatch theoretically. Interesting readers can refer to \cite{ma2022multiscale,ahn2022understanding,arora2022understanding,li2022analyzing,ahn2022learning,zhu2022understanding} for details.

\paragraph{Linear scaling rule of learning rate with batch size.} Linear scaling rule is first proposed by \citep{goyal2017accurate}, suggesting that the when the batch size is smaller than a certain threshold (called critical batch size), scaling the learning rate according to the batch size keep the performance the same. Such a law is further theoretically verified by \cite{ma2018power} which study SGD over quadratic functions. \cite{ma2018power} also show that using linear scaling law above the critical batch size hurt the performance, which is empirically observed by \cite{chen2018effect}.  These methodologies are used in \cite{smithdon} to decay learning rate.

\section{When does momentum accelerate SGD?}
\label{sec:when}

In this section, we explore under what circumstances momentum can accelerate SGD. In Section \ref{subsec: comparison framework}, we first establish a meaningful comparison framework for SGD and SGDM by considering the interplay between momentum and two factors, \ie batch sizes and learning rates. In Section \ref{subsec:momentum-when}, we then conduct experiments under this framework and state our main observations.

\subsection{A comparison framework for SGDMs with different hyperparameters}
\label{subsec: comparison framework}
\textbf{Hyper-parameter scheduler.} We use  constant step-size and constant momentum coefficient across the whole training process, \ie, $\mu_t\equiv \mu$ and $\eta_t\equiv \eta$,  as our primary objective is to understand the acceleration effect of SGDM rather than reproduce state-of-the-art performance.

\textbf{Effective learning rate.} Our aim is to study the essential influence of the momentum coefficient $\mu$ over the performance of SGDM. However, the momentum may affect the performance via different ways. For example, adding $\mu$ will change the update magnitude, which may have the same effect as changing the learning rate. Such effect can be approximated as follows, 
\begin{equation*}
    \bom_t = \sum_{s=1}^t \mu^{t-s} \nabla f_{B_s}(\bw_s)\approx \frac{1-\mu^t}{1-\mu} \nabla f_{B_t}(\bw_t)\rightarrow  \frac{1}{1-\mu} \nabla f_{B_t}(\bw_t)\text{ as } t\rightarrow\infty.
\end{equation*}
This indicates that SGDM with momentum coefficient $\mu$ and learning rate $\eta$ may have the same magnitude of update as SGD with learning rate $\frac{1}{1-\mu}\eta$. When comparing the performances of SGDM with different $\mu$, we want to exclude the effect of momentum that can be compensated  by simply changing the learning rate. Therefore we introduce the concept of \emph{effective learning rate} so that the different setup can be compared fairly to extract the essential effect of momentum.


Additionally, the batch size $b$ is another important hyperparameter in SGDM  whose effect may be compensated by simply changing the learning rate. Specifically, we consider the gradient is averaged (rather than summed) over the individual samples in a minibatch. Larger batch size indicates fewer updates in one epoch. To compensate the number updates in one epoch, we adopt the the Linear Scaling Rule \cite{goyal2017accurate} of the learning rate, which suggests that scaling the learning rate proportionally with the batch size keep the same convergence speed. If the batch size is doubled, doubling the learning rate can keep the convergence speed (with respect to the number of epochs) almost unchanged when the batch size is not too large. This scaling rule has been verified to be effective for models and data with large sizes \cite{goyal2017accurate}. 

Putting these effects together, we propose to compare the performance between SGD and SGDM under the same \emph{effective learning rate}, defined as follows:
\begin{equation*}    \eta^{k}_{\ef}=\underbrace{\frac{1}{1-\mu}}_{\substack{\text{Effect of}\\\text{momentum}}} \cdot \overbrace{\frac{k}{b}}^{\substack{\text{Effect of}\\\text{batch size}}} \cdot \quad \eta,
\end{equation*}
where $k$ is a reference batch size introduced for good visualization for typical choices of batch size and learning rate.  \textbf{When there is no comparison across batch sizes, we simply choose $k=b$, and denote $\eta_{\ef}=\eta^{b}_{\ef}=\frac{1}{1-\mu} \eta$}.

\textbf{Measurement of performance.} As we care about the optimization performance of SGDM, we plot the training loss of SGDM after a prefixed number of epochs $T$, with respect to the effective learning rate for different settings of $\mu$ and $b$ (see Figure \ref{fig:phase_transition}).  We say one setting of SGDM outperforms another, if the former one has a smaller training loss after $T$ epochs for the same effective learning rate.

\subsection{Momentum accelerates training only for large effective learning rates}\label{subsec:momentum-when}

\begin{figure}
\begin{minipage}[c]{0.48\linewidth}
\includegraphics[width=\linewidth]{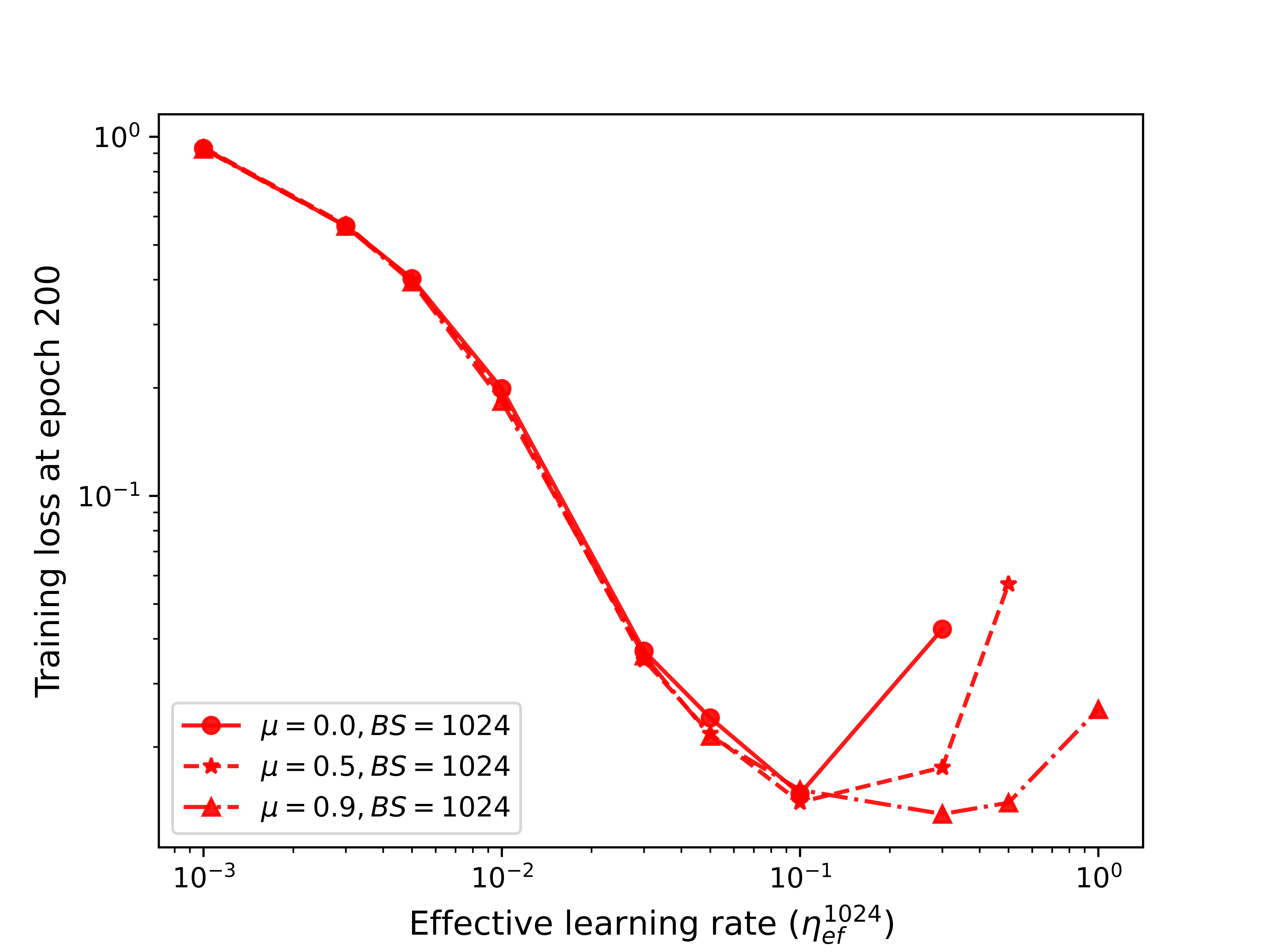}
\caption{\textbf{The training speeds of SGD and SGDM exhibit an align-and-deviate pattern as the effective learning rate  increases.} 1) When the effective learning rate is small, the training speeds of SGDM and SGD are almost the same. 2) After the effective learning rate beyond a certain threshold, SGDM  outperforms SGD.} \label{fig:phase_transition}
\end{minipage}
\hfill
\begin{minipage}[c]{0.48\linewidth}
\includegraphics[width=\linewidth]{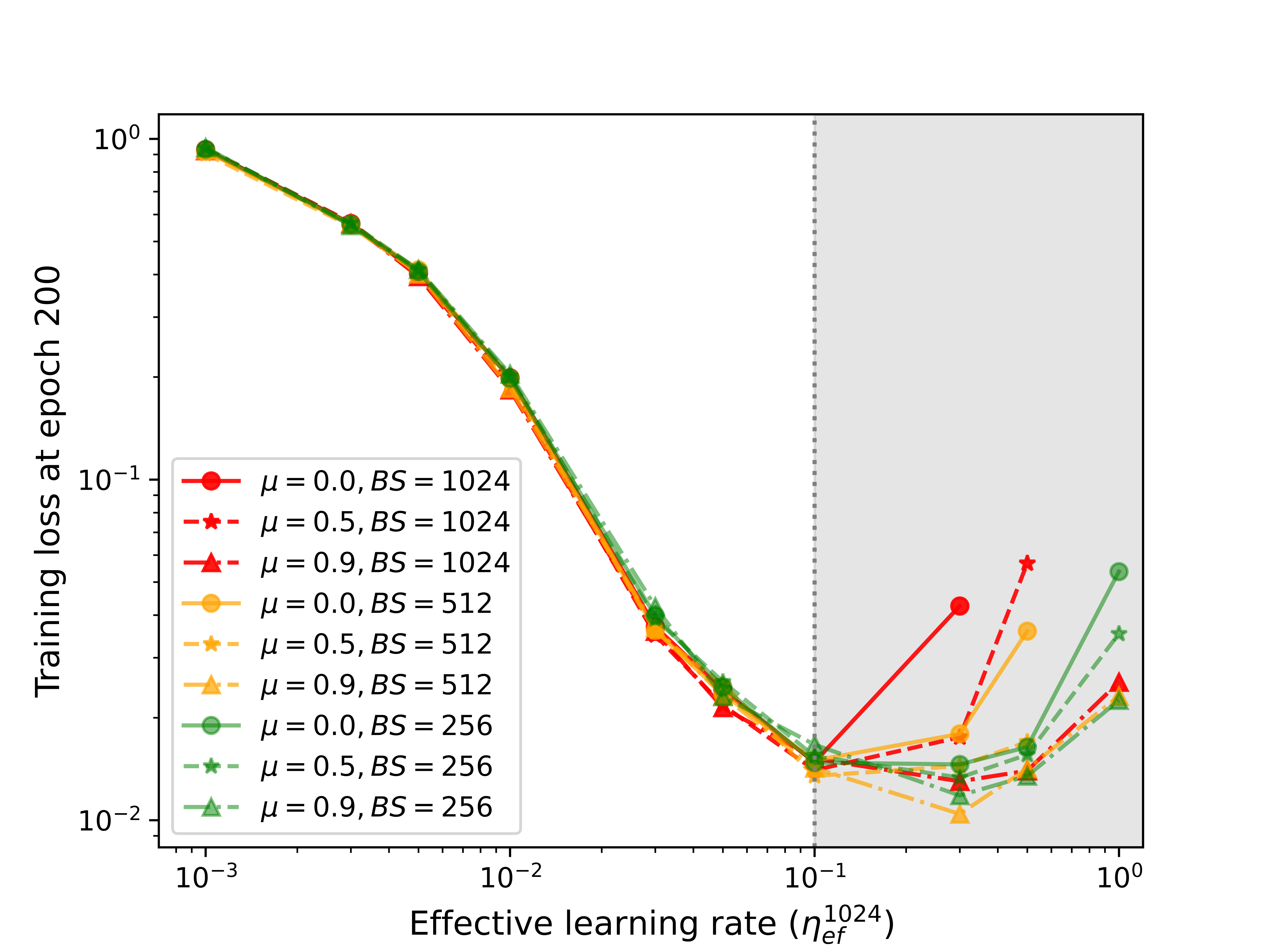}
\caption{\textbf{The benefit of momentum is entangled with batch sizes} 1) Reducing batch size is similar to adding momentum. 2) The gap between SGDM and SGD becomes larger when the batch size increases. 3)The larger momentum coefficient, the smaller gap of SGDMs with different batch sizes.} \label{fig:batchsize-mu}
\end{minipage}%
\end{figure}

We conduct experiments on the CIFAR10 dataset using VGG13-BN network\footnote{More experiments on different architectures and datasets are given in Appendix \ref{sec:dif_dataset_arch}. Our conclusion generally holds.}. We train SGDMs  with batch size 1024 and three values of $\mu=\{0,0.5,0.9\}$, respectively, and choose $k=1024$. We choose the epoch budget $T=200$ (we show in Appendix \ref{sec:dif_dataset_arch} that  our conclusion remains valid regardless of $T$). We note that our findings are also valid for other choices of batch sizes (see Section \ref{subsec:momentum-batchsize-entangle}).

As discussed in Section~\ref{subsec: comparison framework}, we plot curve of  training losses with respect to  effective learning rates for these three settings in Figure \ref{fig:phase_transition}. We summarize the findings as follows.
\begin{itemize}
    \item For small effective learning rates, SGDMs with different values of $\mu$ perform almost the same.  This indicates that momentum does not have the benefit of acceleration because one can always use SGD with a compensated learning rate to reach the performance of SGDM with a specific $\mu$.
    \item As effective learning rate increases beyond some thresholds, the curves with small $\mu$ start deviating from the curve with large $\mu$ progressively, which implies some transition happens. In this regime, we observe the benefit of momentum because simple compensation on learning rate does not helps SGDM with small $\mu$ reach the performance of SGDM with large $\mu$. 
    
\end{itemize}
Overall, for the whole range of $\mu$, SGDM performs better than or equivalent to SGD. 
Such a neat relation can only be observed by introducing effective learning rates to align different values of $\mu$.

In contrast,   \citet{leclerc2020two} present a complicated relation when comparing the performances of SGD and SGDM   under \emph{same learning rates} : when the learning rate is small, SGDM performs better than SGD; however, SGD outperforms SGDM when the learning rate is large. This observation in fact can be explained via the lens of effective learning rate. The key reason is that SGDM experiences  a larger effective learning rate  $\frac{\eta}{1-\mu}$ than that of SGD $\eta$ when their learning rates are the same. This is beneficial for SGDM with small learning rates.  As the learning rate increases, the performance of SGD consistently improves but  the performance of SGDM first improves and then starts to drop after surpassing a sweet point of the effective learning rate as seen in Figure \ref{fig:phase_transition}. 

\subsection{The benefit of momentum is entangled with batch sizes}\label{subsec:momentum-batchsize-entangle}

In this section, we examine the effect of  batch size $b$ and understand how different batch sizes affect the benefit of momentum. We repeat the experiments in Section \ref{subsec:momentum-when}, i.e, experiments on the CIFAR10 dataset using VGG13-BN network for SGDM with $9$ representative choices of hyperparameters $(\mu,b)\in \{0,0.5,0.9\}\times \{256,512,1024\}$. We choose $k=1024$ and the epoch budget $T=200$, the same as before. We plot the result of the training losses with respect to the effective learning rates in Figure \ref{fig:batchsize-mu}. 
 
Our findings are summarized as follows.

\begin{itemize}
    \item \textbf{Reducing batch size is similar to adding momentum.} When comparing SGDs ($\mu=0$) with different batch sizes, we observe that SGDs perform almost the same for small effective learning rates regardless of batch sizes; when the effective learning rate is above a certain threshold, SGD with smaller batch size achieve smaller training loss.  This shows that reducing the batch size has a similar pattern to adding the momentum. Moreover, combined with the result in Section \ref{subsec:momentum-when}, as long as the effective learning rates are in a regime with small values, SGDMs performs almost the same regardless of changing the batch size or the momentum coefficient.

    \item \textbf{The gap between SGDM and SGD becomes larger when the batch size increases.} This coincides with  the  observation that the acceleration effect of momentum is more pronounced  with larger batch sizes  \cite{kunstnernoise,kidambi2018insufficiency,shallue2019measuring}.

    \item \textbf{The larger momentum coefficient, the smaller gap of SGDMs with different batch sizes.} As one increases the momentum coefficient, the gap between the performances of SGDMs with different batch sizes becomes smaller. 

\end{itemize}

To see the last point more clearly, we conduct additional experiments. Specifically, we fix  the effective learning rate $\eta_{\ef}^{1024}=0.1$ and  gradually increase the batch size  to plot a curve of training loss with respect to the batch size in Figure \ref{fig:batch_size_scaling}. We observe that when batch sizes are small, SGDM with different $\mu$s performs almost the same, and when the batch size increases beyond a threshold, SGDM with larger $\mu$ tends to perform better. 

Moreover, we note that a horizontal curve in Figure \ref{fig:batch_size_scaling} is equivalent to the Linear Scaling Law \cite{goyal2017accurate} and we can see that \textbf{momentum extends the range of batch sizes in which the Linear Scaling Law holds.} 

\begin{figure}[h]
    \centering
    \includegraphics[width=0.70 \textwidth]{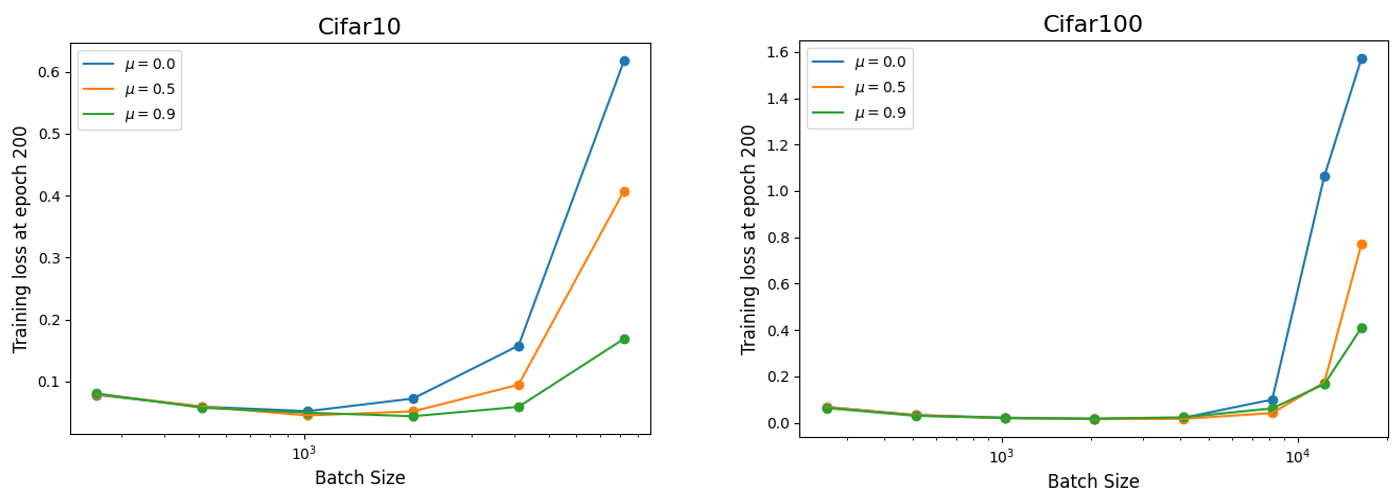}
    \caption{\textbf{Momentum extends the range of batch sizes in which the Linear Scaling Law holds.} 1) A similar align-and-deviate  pattern for SGD and SGDM is also discovered when the batch size is increased. 2) Each curve remains nearly horizontal up to a specific threshold batch size. The threshold batch size for SGDM is greater than that for SGD.}
    \label{fig:batch_size_scaling}
\end{figure}

\section{Why does momentum accelerate SGD?}
In Section \ref{sec:when}, we have explored when momentum accelerates SGD. In this section, we want to understand why momentum accelerates SGD, or more precisely, the mechanism of momentum accelerating SGD. 

To control variable and simplify the analysis, we first focus on comparing GD and GDM, and find that deviation of GD and GDM is related with a phenomenon that \emph{Hessian abruptly sharpens} along the update direction in Section \ref{subsec:eos_phasetransition}. We then show that the abrupt sharpening of Hessian can explain the acceleration of momentum and batch size in Section \ref{subsec: abrupt_effect} and \ref{subsec:small_batch_training}, respectively.

\subsection{Hessian abruptly sharpens when GD deviates from GDM}
\label{subsec:eos_phasetransition}

As the full-batch update is computationally expensive, we use a subset of CIFAR10 with 5K samples, which has been used to  study  GD behavior previously \cite{cohen2021gradient, ahn2022understanding}. The network we use is  fc-tanh\cite{cohen2021gradient}, i.e., a one-hidden-layer fully-connected network with 200 neurons and tanh activation.

We first verify that the align-and-deviate pattern still exist in this task in Figure \ref{fig:gd}A. We then pick one effective learning rate $\eta_{\ef}=0.01$ before the deviation threshold  and one effective learning rate $\eta_{\ef}=0.1$ after the threshold and plot its training loss across epochs in Figure \ref{fig:gd}B and Figure \ref{fig:gd}C. We can see that for $\eta_{\ef}=0.01$, the training curves of GDM and GD are smooth and closely aligned (Figure \ref{fig:gd}B). When  $\eta_{\ef}=0.1$ beyond the threshold in Figure \ref{fig:gd}A, the training curves of GDM and GD align with each other in first few epochs and then the loss of GD starts oscillating and deviates from the loss of GDM. It should be noted that GD becomes slower than GDM, i.e, the curve of GD is strictly on top of that of GDM, at the same moment that GD starts oscillating.



\begin{figure}[h]
    \centering
    \includegraphics[width=0.7 \textwidth]{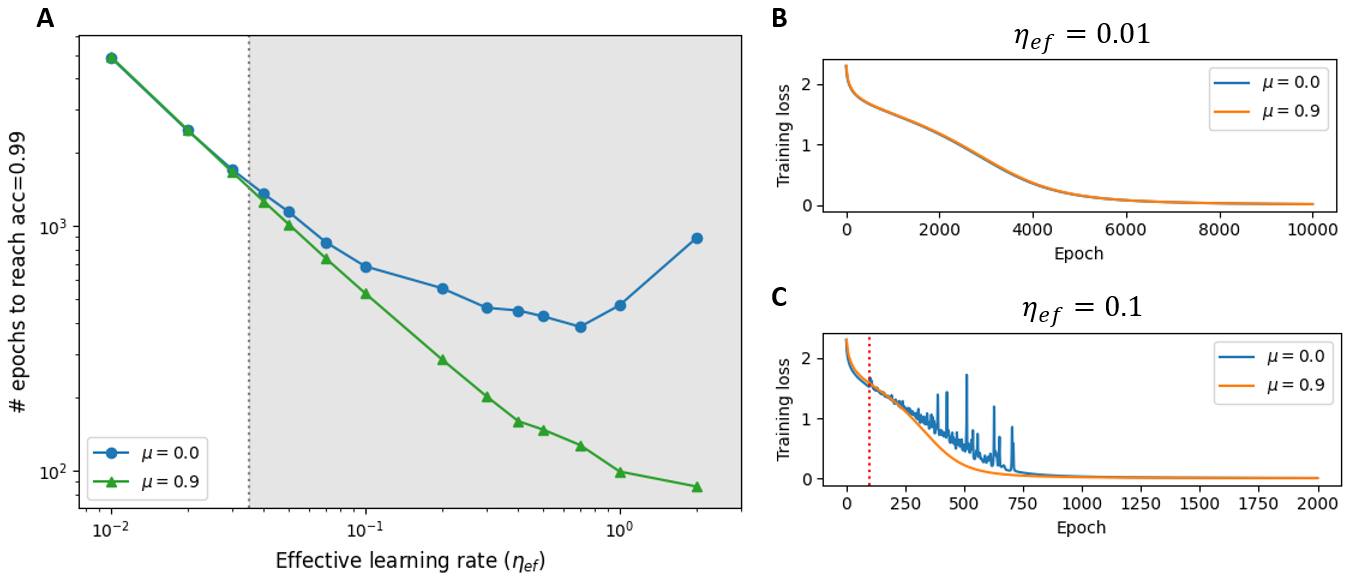}
    \caption{\textbf{Exploration of the training process on CIFAR10-5k dataset.} A: Experiments on Cifar10-5k  gives a similar result as Figure \ref{fig:phase_transition}. B: GD and GDM are aligned during the whole training process under small learning rate. C: GD and GDM are aligned before GD starting oscillating, and deviate after. The red dash line denotes the time when the GD starts oscillating.}
    \label{fig:gd}
\end{figure}

The observation in Figure \ref{fig:gd} provides a partial answer of why momentum accelerates GD by connecting it with preventing oscillation. However, we are still unclear the reason for why the oscillations happen and why GDM is less like to oscillate. With $\bw_t$ be the iteration of GD, we revisit the Taylor expansion of the objective function $f$, which writes
\begin{equation*}
    f(\bw_{t+1})\approx f(\bw_{t})-\eta \Vert  \nabla f(\bw_t)\Vert^2 +\frac{\eta^2}{2} \nabla f(\bw_t)^{\top}\nabla^2 f(\bw_t) \nabla f(\bw_t). 
\end{equation*}
When the loss stably decreases, we have $f(\bw_{t+1})< f(\bw_{t})$, and based on the above approximation, we infer the \emph{directional Hessian} along the update direction $H(\bw_t,\bw_{t+1}-\bw_t)\triangleq \frac{(\bw_{t+1}-\bw_t)^{\top}\nabla^2 f(\bw_t) (\bw_{t+1}-\bw_t)}{ \Vert \bw_{t+1}-\bw_t\Vert^2}$ satisfying $\eta H(\bw_t,\bw_{t+1}-\bw_t)< 2$. On the other hand, we have $f(\bw_{t+1})\approx f(\bw_{t})$ when oscillating \cite{ahn2022understanding}, and simple calculation gives $\eta H(\bw_t,\bw_{t+1}-\bw_t)\approx 2$. Therefore, we conjecture that the oscillation is due to a sharp transition of the  \emph{directional Hessian} along the update. 

To verify our conjecture, we plot the directional Hessian  in Figure \ref{fig:gdhess}.  We observe that there is a sharp transition of the directional Hessian along the update: 
it first stays around $0$ before oscillation, and then experiences a sudden jump at the time of oscillation. We referred to this phenomenon as ``directional Hessian abrupt sharpening'' or ``abrupt sharpening'' for short.

\begin{figure}[h]
    \centering
    \includegraphics[width=0.95 \textwidth]{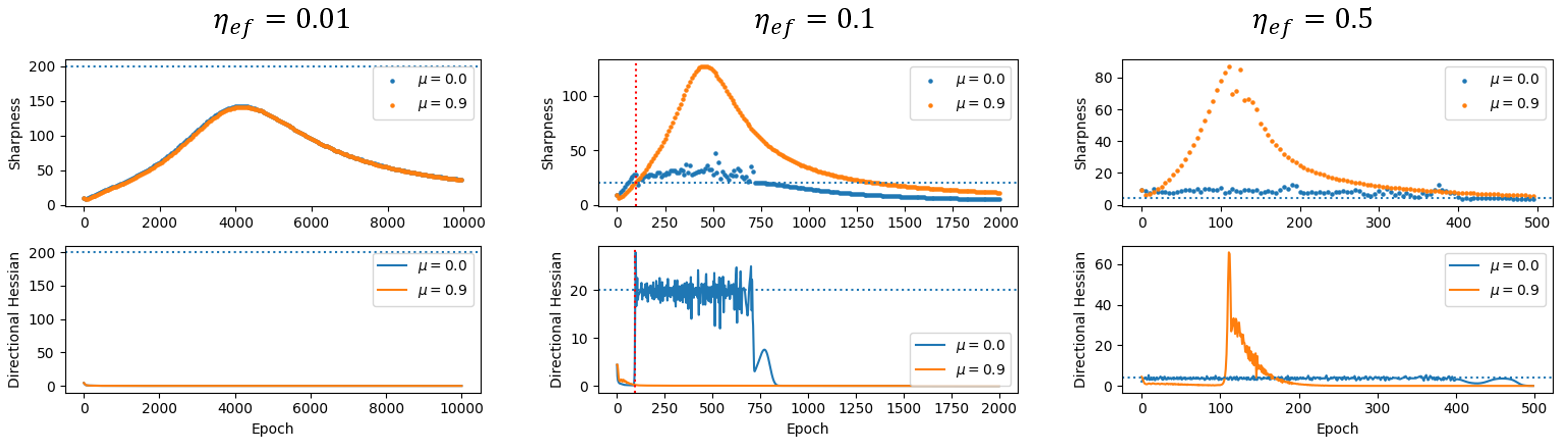}
    \caption{\textbf{Sharpness and directional Hessian on CIFAR10-5k dataset.} The dashed blue line represents the threshold  $\frac{2}{\eta_{\ef}}$.  \textit{Left:} With a small effective learning rate, the directional Hessian of GD and GDM are around $0$.  \textit{Center:} With a larger effective learning rate, GD exhibits abrupt sharpening and training loss starts to oscillate (marked by the red dash line) when the sharpness of GD surpasses $\frac{2}{\eta_{\ef}}$, while directional Hessian of GDM stays around $0$. \textit{Right:} With an even larger effective learning rate, both GDM and GD exhibits abrupt sharpening, but much later for GDM .}
    \label{fig:gdhess}
\end{figure}
Abrupt sharpening explains why the oscillations happen.  We further notice that abrupt sharpening is closely related to an existing concept called "edge of stability" (EoS), which describes the phenomenon that during the training of GD, sharpness, i.e., the maximum eigenvalue of Hessian, will gradually increase until it reaches $\frac{2}{\eta}$ and then hover at it. The phenomenon of gradually increasing sharpness is denoted as "progressive sharpening". It seems to contradict with abrupt sharpening of directional Hessian, but we show that abrupt sharpening is a joint outcome of progressive sharpening and renowned degenerate Hessian of deep neural networks \cite{sagun2016eigenvalues} through the following proposition. Consequently, abrupt sharpening can be viewed as a newfound component of EoS.
\begin{proposition}
\label{prop: convergence_direction}
Given a minimization problem $\min_{\bw\in \mathbb{R}^d} f(\bw)$, we consider minimizing its quadratic function approximation  around a minimizer $\bw^*$, i.e., $\tilde{f}(\bw)\triangleq \frac{1}{2}(\bw-\bw^*)^{\top} \nabla^2f(\bw^*) (\bw-\bw^*)+f(\bw^*)$. Let $\bw_t$ be the parameter given by GD with learning rate $\eta_{ef}$ at the $t$-th iteration. Let $\mathcal{A}$ be the space of eigenvectors of $\nabla^2f(\bw^*)$ corresponding to the maximum eigenvalue of $\nabla^2f(\bw^*)$. For almost every $\bw_0\in \mathbb{R}^d$ , $\lim_{t\rightarrow \infty}\frac{\nabla \tilde{f}(\bw_t)}{\Vert \nabla \tilde{f}(\bw_t) \Vert} \in \mathcal{A} $ if and only if $\lambda_{\max}(\nabla^2f(\bw^*))> \frac{2}{\eta_{\ef}}-\lambda_{\min} (\nabla^2f(\bw^*))$.
\end{proposition}
\citet{sagun2016eigenvalues} observes that in deep learning tasks, the smallest eigenvalue of Hessian $\lambda_{\min} (\nabla^2f(\bw^*))$ is close to $0$. This together with Proposition \ref{prop: convergence_direction} indicates that the update direction of GD would start to align with the eigenspace of the maximum eigenvalue only after the sharpness is very close to $\frac{2}{\eta}$ (thus in the early stage, directional sharpness stays around $0$). Once such an alignment starts, it is rapid because the convergence rate in Proposition \ref{prop: convergence_direction} is exponential (please see the proof in Appendix \ref{sec:proof} for details), which explains the abrupt sharpening of directional Hessian.

\subsection{Abrupt sharpening can explain the acceleration of momentum}
\label{subsec: abrupt_effect}
Here we show that abrupt sharpening can be used to explain the acceleration of momentum.

\textbf{Small directional Hessian explains the alignment between GD and GDM.} Intuitively, when directional Hessian is relatively small, GD and GDM are like walking straightly on a line because small directional Hessian implies small change of gradient along the update direction. This agrees with the setting where we introduce effective learning rate, i.e., $\nabla f(\bw_1)\approx \nabla f(\bw_2)\approx \cdots \approx \nabla f(\bw_t)$, and thus GDM performs similarly as GD under the same effective learning rate. This perfectly explains the alignment between GD and GDM before oscillation. We summarize the above intuition as the following property.
\begin{proposition}
\label{prop:gdgdm}
Denote the iterations of GD as $\{\bw_t^{\GD}\}_{t=1}^{\infty}$ and those of GDM as $\{\bw_t^{\GDM}\}_{t=1}^{\infty}$. If the directional Hessians satisfy $H(\bw_s^{\GD}, \bw_{s+1}^{\GD}-\bw_{s}^{\GD})\approx 0$ and $H(\bw_s^{\GDM}, \bw_{s+1}^{\GDM}-\bw_{s}^{\GDM})\approx 0$, $\forall s\le t-1$, then we have
$
    f(\bw_t^{\GD})\approx f(\bw_t^{\GDM}).
$
\end{proposition}

\textbf{Momentum defers abrupt sharpening, and thus accelerates GD.} First of all, we show again through quadratic programming that momentum has the effect to defer abrupt sharpening.

\begin{proposition}
\label{prop: convergence_direction_momentum}
Let $f$, $\bw^*$ and  $\tilde{f}$ and $\mathcal{A}$ be defined in Proposition \ref{prop: convergence_direction_momentum}. Let $\bw_t$ be the parameter given by GDM at the $t$-th iteration.  Then, for almost everywhere $\bw_0\in \mathbb{R}^d$, $\lim_{t\rightarrow \infty}\frac{\nabla \tilde{f}(\bw_t)}{\Vert \nabla \tilde f(\bw_t) \Vert} \in \mathcal{A} $ if and only if $\lambda_{\max}(\nabla^2 f(\bw))> \frac{2(1+\mu)}{(1-\mu)\eta_{\ef}}-\lambda_{\min} (\nabla^2 f(\bw))$.
\end{proposition}
Comparing Proposition \ref{prop: convergence_direction_momentum} with Proposition \ref{prop: convergence_direction}, we observe that with a relative small $\lambda_{\min}$, the required $\lambda_{\max}$ for abrupt sharpening appearance of GDM is $\frac{(1+\mu)}{1-\mu}$ (which is $19$ when $\mu=0.9$) times larger than that of GD. As the sharpness progressively increases, reaching the required sharpness of GDM takes a much longer time than reaching that of GD (an extreme case is that abrupt sharpening happens in GD but not in GDM). Meanwhile, entrance of edge of stability has been known to slow down the convergence. In \cite{ahn2022learning}, it is shown that when not entering EoS, GD converges in $\mathcal{O}(1/\eta_{\ef})$ iterations, but require $\Omega(1/\eta_{\ef}^2)$ iterations to converge in the EoS regime. Together, we arrive at the conclusion that momentum can accelerate GD via deferring the entrance of EoS (abrupt sharpening).

\subsection{Extending the analysis to stochastic case: interplay between momentum and batch size}
\label{subsec:small_batch_training}
 Over the same experiment of Sections \ref{subsec:eos_phasetransition} and \ref{subsec: abrupt_effect}, we first plot the training curves of GD, GDM, and SGD with batch size $250$ and $\eta^{5000}_{\ef}=0.1$ in Figure \ref{fig:CIFAR10_5k_SGD}A. Specifically, we find that  \textbf{stochastic noise can also defer abrupt sharpening}: GD enters EoS during the training process, while SGD and GDM does not and they remain well-aligned throughout the training process.

Since in previous section, we have explained that entrance of EoS slows down the convergence, such an observation explains why in Figure \ref{fig:phase_transition}, reducing batch size also accelerates SGDM with respect to the number of passes of the data. Furthermore, this observation also explains why the effect of momentum is more pronounced when batch size is large since stochastic noise and momentum has an overlapping effect in preventing abrupt sharpening.
\begin{figure}[h]
    \centering
    \includegraphics[width=0.95 \textwidth]{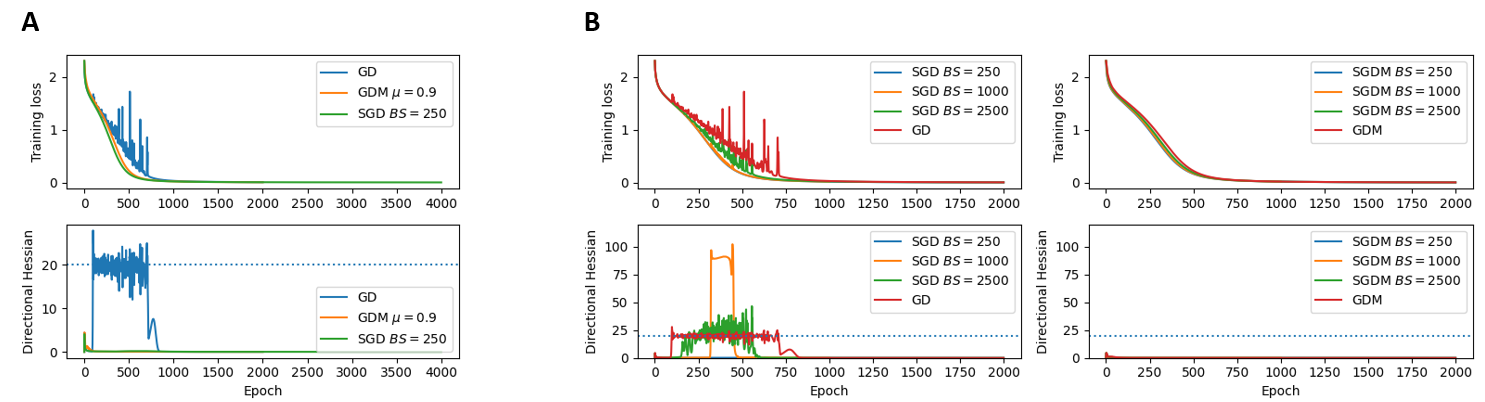}
    \caption{\textbf{Reducing  batch size and adding momentum play a similar role in preventing abrupt sharpening.} A: Reducing  batch size can help prevent abrupt sharpening. B: Adding momentum can extend the range of batch sizes where linear scaling rule holds.}
    \label{fig:CIFAR10_5k_SGD}
\end{figure}

From Figure \ref{fig:CIFAR10_5k_SGD}B we can see that  the performance of large-batch SGD is worse than small-batch SGD because large-batch enters EoS while small-batch does not. When momentum is added, large-batch also does not enter EoS and hence the range of linear scaling law is extended.

\section{Conclusion}
This paper investigates the relationship between momentum, learning rate, and batch size. 
We observe an align-and-deviate pattern when either fixing the batch size and increasing the effective learning rate (Figure \ref{fig:phase_transition}) or fixing the effective learning rate and increasing the batch size (Figure \ref{fig:batch_size_scaling}). Before the deviation point, the training speed of SGD and SGDM are almost the same. However, after the deviation point, SGDM outperforms SGD. We link the phase transition to the EoS and explain that momentum accelerates training via preventing the entrance of EoS.   We also observe and analyze the effect of batch size following the above framework. In summary, this paper provides thorough empirical result to see and analyze when and why momentum accelerates SGD under various settings.  

\section{Limitation}
Our current paper has two limitations. Firstly, the assessment of model architectures and datasets is not comprehensive. In the Appendix \ref{sec:dif_dataset_arch}, we perform experiments on commonly used model architectures and popular datasets. Our findings are based on these configurations. Though it is never possible to test all scenarios, more extensive experiments are required in future work. Secondly, our main focus is on the acceleration of momentum in SGD and we do not cover the widely used optimizer   Adam, which has a more complex analysis than SGDM due to its adaptive learning rate. We defer the study of Adam to future work.

\bibliography{main}
\bibliographystyle{abbrvnat}

\appendix
\newpage

\section{Further experiments on other optimizers}
\label{sec:Nesterov}

\begin{figure}[h]
    \centering
    \includegraphics[width=0.6 \textwidth]{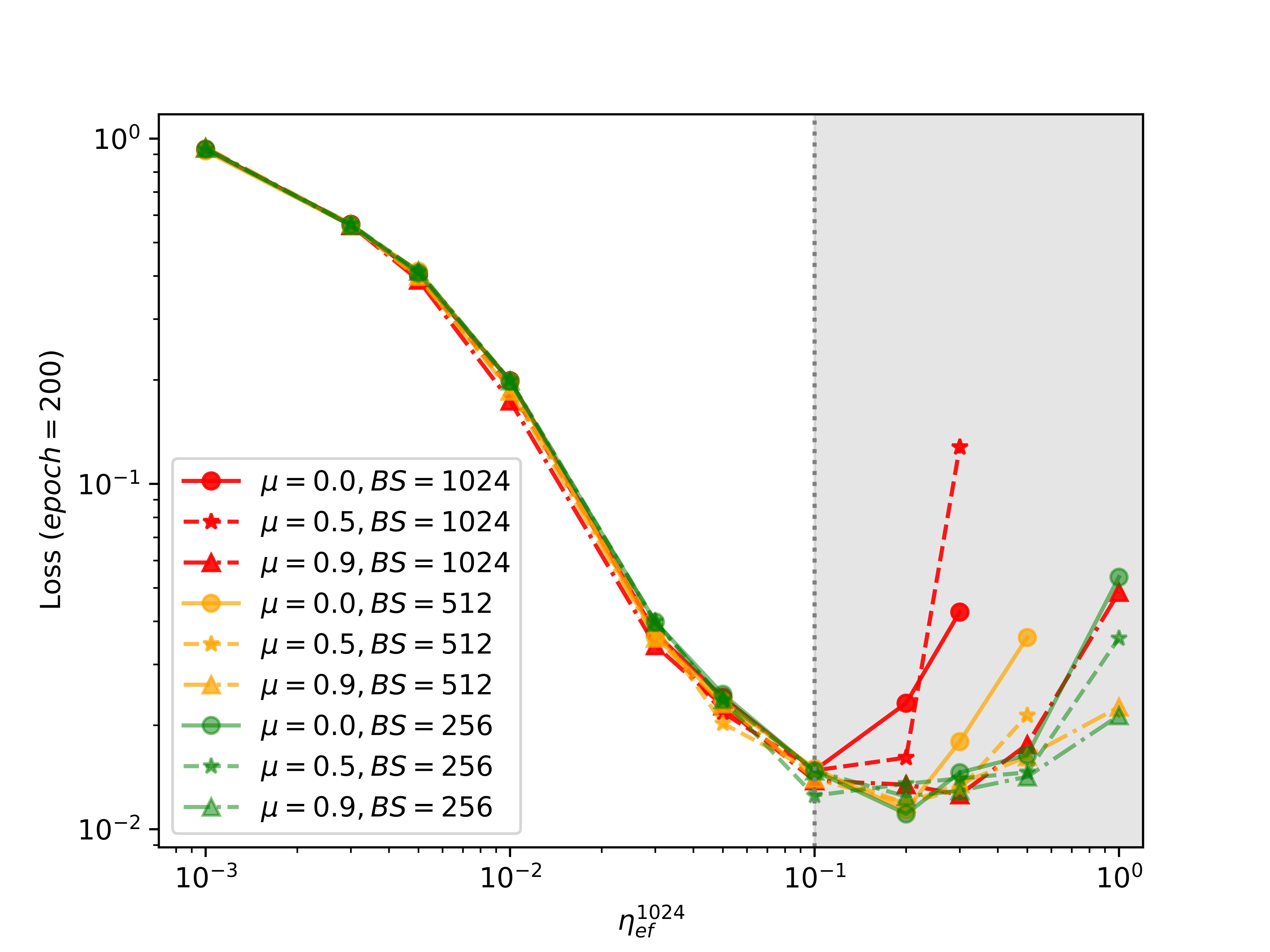}
    \caption{\textbf{The align-and-deviate pattern for Nesterov's Momentum.} The behavior of Nesterov's momentum is similar to that of Polyak's momentum (see Figure \ref{fig:batchsize-mu}) when increasing the effective learning rate.}
    \label{fig:nesterov_phase_transition}
\end{figure}

\subsection{On the effect of Nesterov's momentum}
To give a full picture of the effect of momentum, we further conduct experiments over SGD with Nesterov's momentum as a complement to the discussion about the Polyak's momentum in  the main text. Specifically, the update rule of Nesterov's momentum is given in Algorithm \ref{alg:nesterov}, which is the implementation in PyTorch.

\begin{algorithm}
    \centering
    \caption{SGD  with Nesterov's Momentum}\label{alg:nesterov}
    \hspace*{0.02in} 
    \begin{algorithmic}[1]
        \State \textbf{Input:} the loss function $\ell(w,z)$, the initial point $\bw_{1} \in \mathbb{R}^d$,  the batch size $b$, learning rates $\{\eta_t\}_{t=1}^{T}$, {  $\bom_0=0$, and momentum hyperparameters $\{\mu_t\}_{t=1}^{T}$}.
        \State \textbf{For} $t=1\rightarrow T$:
        \State ~~~~~Sample a  mini-batch of data $B_t$ with size $b$ 
       \State ~~~~~Calculate stochastic gradient $\nabla f_{B_t}(w_t)=\frac{1}{b}\sum_{z\in B_t}\ell (w_t,z)$
       \State ~~~~~Update $\bom_{t}\leftarrow${$\mu_t \bom_{t-1}+$}$\nabla  f_{B_t}(\bw_{t})$
        \State ~~~~~Update $\bw_{t+1}\leftarrow\bw_t-\eta_t (\mu_t \bom_{t}+\nabla  f_{B_t}(\bw_{t}))$
        \State \textbf{End For}
    \end{algorithmic}
\end{algorithm}
    

\subsubsection{Derivation of effective learning rate}
Like in the analysis of Polyak's momentum, we fix $\eta_t$ and $\mu_t$ to be constants. We show below that Nesterov's momentum has the similar effect as Polyak's momentum to amplify the update magnitude. Specifically, we have
\begin{equation*}
    \bom_t = \sum_{s=1}^t \mu^{t-s} \nabla f_{B_s}(\bw_s)\approx \frac{1-\mu^t}{1-\mu} \nabla f_{B_t}(\bw_t)\rightarrow  \frac{1}{1-\mu} \nabla f_{B_t}(\bw_t)\text{ as } t\rightarrow\infty,
\end{equation*}
and thus
\begin{equation*}
    \mu \bom_{t}+\nabla  f_{B_t}(\bw_{t})\approx\frac{1}{1-\mu}  \nabla f_{B_t}(\bw_t)\text{ as } t\rightarrow\infty.
\end{equation*}
To rule out such a effect, we define the effective learning rate of Nesterov's momentum as
\begin{equation*}    
\eta^{k}_{\ef}=\frac{1}{1-\mu} \cdot \frac{k}{b} \cdot \eta.
\end{equation*}

\subsubsection{Experiments}
\begin{figure}[h]
    \centering
    \includegraphics[width=0.95 \textwidth]{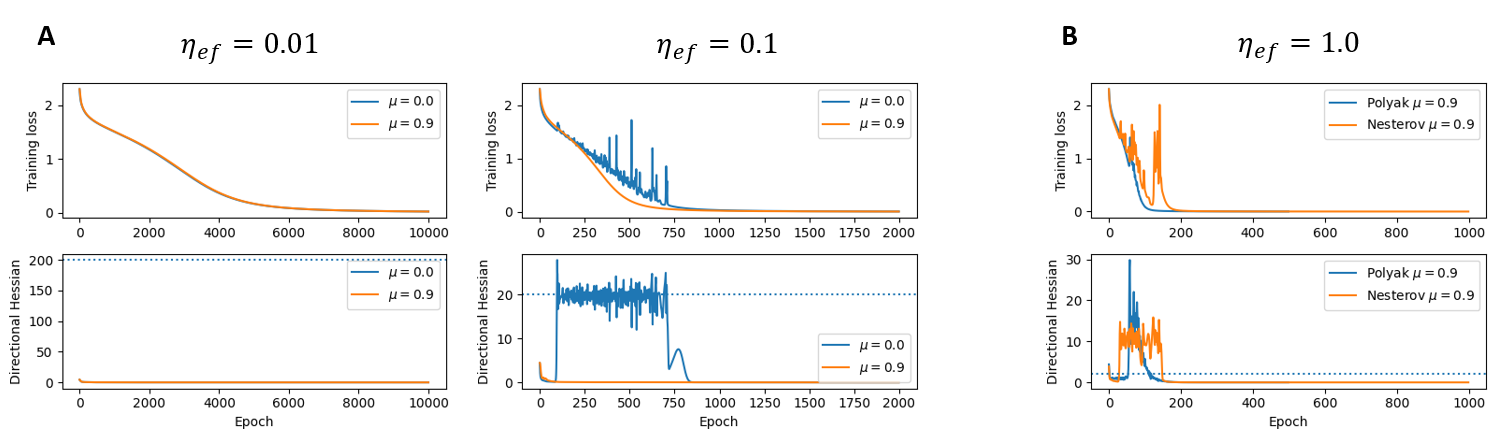}
    \caption{\textbf{Exploration of Nesterov's momentum} A: Nesterov's can also prevent the abrupt sharpening. B: Compared with Polyak, Nesterov's performs worse in preventing abrupt sharpening. Nesterov's GDM enters EoS earlier than Polyak's momentum. Additionally, the training speed of Nesterov's momentum is slower.}
    \label{fig:nesterov_abrupt_shapening}
\end{figure}

We conduct the experiments of SGD with Nesterov's momentum under the same setup as Figure~\ref{fig:batchsize-mu}. We plot the results in Figure \ref{fig:nesterov_phase_transition}. We can see that optimizers with Nesterov's momentum behave similarly to the counterparts with Polyak's momentum as shown in Figure \ref{fig:batchsize-mu}. Furthermore, we provide a further investigation on Nesterov's momentum by conducting an experiment based on the setup of Figure \ref{fig:gd}, plotted in Figure \ref{fig:nesterov_abrupt_shapening}. Figure \ref{fig:nesterov_abrupt_shapening}A shows that the Nesterov momentum can also prevent abrupt sharpening during the training process. Then, we give a simple comparison between Polyak's and Nesterov's momentum by comparing them together under $\eta_{\ef}=1.0$, where both of them will enter the EoS. In this setting, we find that compared with SGD with Polyak's momentum, SGD with Nesterov's momentum with same $\mu$ entere EoS earlier (Figure \ref{fig:nesterov_abrupt_shapening}B), and Polyak's momentum performs better than Nesterov's momentum under this setting.
However, future work with more extensive experiments is required before making any conclusive claim on optimizers with Nesterov's momentum. In this paper, we focus on the optimizers with Polyak's momentum. 


\subsection{On the effect of momentum in  Adam}
\begin{algorithm}
    \centering
    \caption{Adam}\label{alg:adam}
    \hspace*{0.02in} 
    \begin{algorithmic}[1]
        \State \textbf{Input:} the loss function $\ell(w,z)$, the initial point $\bw_{1} \in \mathbb{R}^d$,  the batch size $b$, learning rates $\{\eta_t\}_{t=1}^{T}$, {  $\bom_0=0$,$\bv=0$, and hyperparameters $\beta=(\beta_1,\beta_2)$}.
        \State \textbf{For} $t=1\rightarrow T$:
        \State ~~~~~Sample a  mini-batch of data $B_t$ with size $b$ 
       \State ~~~~~Calculate stochastic gradient $\nabla f_{B_t}(w_t)=\frac{1}{b}\sum_{z\in B_t}\ell (w_t,z)$
       \State ~~~~~Update $\bom_{t}\leftarrow${$\beta_1 \bom_{t-1}+$}$(1-\beta_1)\nabla  f_{B_t}(\bw_{t})$
        \State ~~~~~Update $\bv_{t}\leftarrow${$\beta_2 \bv_{t-1}+$}$(1-\beta_2)\nabla  f_{B_t}(\bw_{t})^{\odot 2}$
        \State ~~~~~Update $\bw_{t+1}\leftarrow\bw_t-\eta_t \frac{\bom_t/(1-\beta_1^t)}{\sqrt{\bv_t/(1-\beta_2^t)+\epsilon}}$
        \State \textbf{End For}
    \end{algorithmic}
\end{algorithm}
Here we step beyond SGD and provide a preliminary investigation on the effect of momentum in Adam \cite{kingma2014adam}.  The psedocode of Adam is given in Algorithm \ref{alg:adam}. We first derive the effective learning rate of Adam. Since
\begin{equation*}
    \bom_t = (1-\beta_1)\sum_{s=1}^t \beta_1^{t-s} \nabla f_{B_s}(\bw_s)\approx (1-\beta^t) \nabla f_{B_t}(\bw_t)\rightarrow \nabla f_{B_t}(\bw_t)\text{ as } t\rightarrow\infty,
\end{equation*}
we define the effective learning rate of Adam directly as the learning rate \ie $\eta_{\ef}=\eta$ (here we do not consider the effect of batch size since it is still an open problem for the effect of batch size in Adam).
We conduct the experiments of full-batch Adam under the same setup as Figure~\ref{fig:gd}. Since we focus on the effect of momentum, we fix $\beta_2=0.999$ (which is the default value in PyTorch) and choose $\beta_1$ from $\{0,0.5,0.9\}$. The results are plotted in Figure \ref{fig:adam}.

\begin{figure}[h]
    \centering
    \includegraphics[width=0.5 \textwidth]{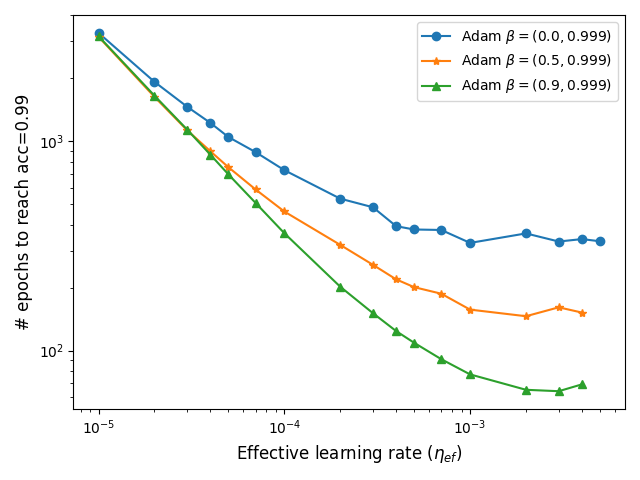}
    \caption{\textbf{The align-and-deviate pattern also exists in Adam.} When increasing the effective learning rate, Adam with different $\beta_1$ also exhibits an align-and-deviate pattern. Here $\beta=(\beta_1,\beta_2)$.}
    \label{fig:adam}
\end{figure}
\section{More exploration on the align-and-deviate pattern}
\label{sec:dif_dataset_arch}

\paragraph{Influence of model architecture.}
In this study, we investigate whether varying model designs have an impact on the final conclusions. We set the batch size to 1024 and allocate an epoch budget of T = 200. The experiments are carried out using the Cifar10 dataset. All the considered architectures exhibit the align-and-deviate pattern. However, the effect of momentum varies across different models. For instance, momentum plays a more significant role in improving performance for VGG13 than that for VGG13BN as observed in Figure \ref{fig:model}. Moreover, momentum is particularly important for training the ViT\citep{dosovitskiy2020image} model, as depicted in Figure \ref{fig:model} (ViT).

\begin{figure}[h]
    \centering
    \includegraphics[width=1 \textwidth]{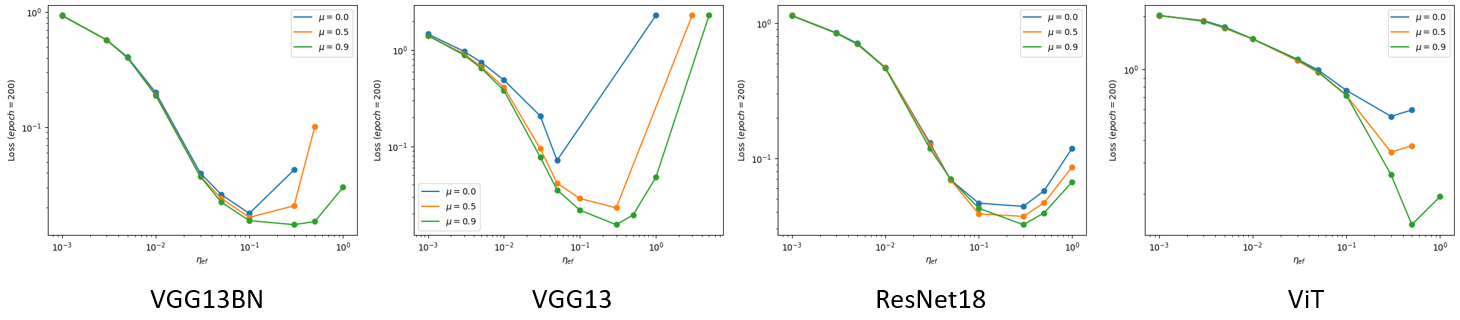}
    \caption{\textbf{Experiments with different neural network architectures.} Momentum has a more significant role in VGG13 and ViT network compared with ResNet18\cite{he2016deep} and VGG13BN network.}
    \label{fig:model}
\end{figure}

\paragraph{Influence of datasets}The experiments, as illustrated in Figure \ref{fig:different_datasets}, are carried out using a variety of datasets, such as Cifar100\cite{simonyan2014very}, WikiText2\cite{merity2016pointer}, and ImageNet\cite{deng2009imagenet}. For each dataset, we employ a different model: VGG13BN for Cifar100, Transformer\footnote{\url{https://pytorch.org/tutorials/beginner/transformer_tutorial.html}} for WikiText2, and ResNet18 for ImageNet. We consistently observe the align-and-deviate pattern across these datasets. However, the positions of deviation points differ considerably among them. This variation could be attributed to factors such as dataset size, task difficulty, and other aspects.

\begin{figure}[h]
    \centering
    \includegraphics[width=1 \textwidth]{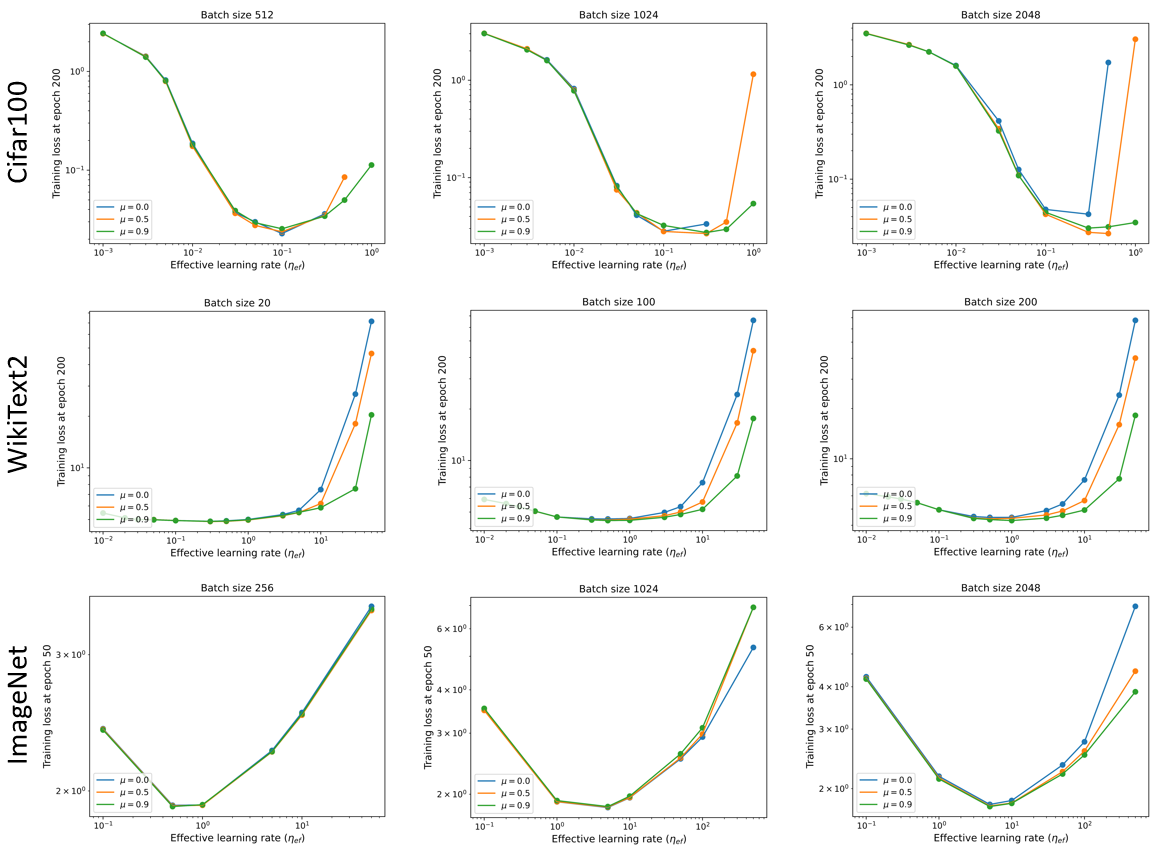}
    \caption{\textbf{Experiments across different datasets.} The align-and-deviate pattern is consistently observed along these datasets.}
    \label{fig:different_datasets}
\end{figure}

\paragraph{Different epoch setting.} In this paper, we use the training loss at epoch $T$ to represent the training speed of optimizers. The $T$ is chosen to be 200 in our experiments. Here, we explore different values of T from $\{50,100,150,200 \}$, and we want to check whether the choice of  $T$ matters. From Figure \ref{fig:epoch}, we observe  that the align-and-deviate pattern exists no matter what value of $T$ is chosen.

\begin{figure}[h]
    \centering
    \includegraphics[width=0.95 \textwidth]{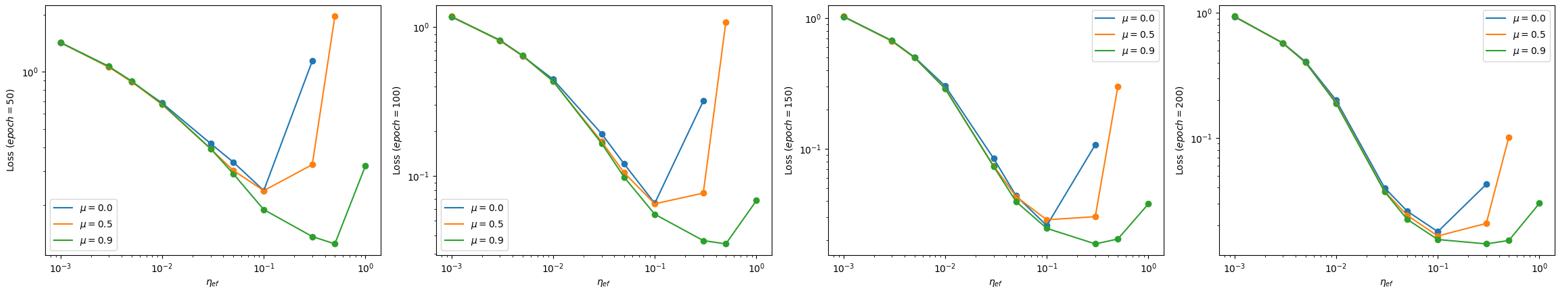}
    \caption{\textbf{Exploration the align-and-deviation pattern with different epoch settings.} The align-and-deviate pattern is observed in all these settings. }
    \label{fig:epoch}
\end{figure}

\section{Proofs of theoretical results}
\label{sec:proof}

\subsection{Proof of Propositions \ref{prop: convergence_direction} and \ref{prop: convergence_direction_momentum}}
\label{appen: direction}
\begin{proof}[Proof of Proposition \ref{prop: convergence_direction}]
To begin with, by linear transformation, we can assume without loss of generality that $\bA$ is a diagonal matrix, $\boldsymbol{b}=0$ and $c=0$. Denote $\bA=\operatorname{Diag}(\lambda_1,\cdots,\lambda_d)$, where $\lambda_{\max}(\bA)=\lambda_1\ge \lambda_2 \ge \dots \ge \lambda_d= \lambda_{\min}(\bA)$. Denote $\bw_t=(\bw_{t,1},\cdots,\bw_{t,d})$. Let $m_1$ be the number of eigenvalues equal to $\lambda_{\max}$. Let $m_2$ be the number of eigenvalues equal to $\lambda_{\min}$. Then, $\mathcal{A}=\operatorname{span}\{\boldsymbol{e}_1,\cdots, \boldsymbol{e}_{m_1}\}$. Based on the update rule of GD, we obtain that
\begin{equation*}
    \bw_{t+1,i}=(1-\eta \lambda_i) \bw_{t,i},
\end{equation*}
and thus $\bw_{t,i}=(1-\eta \lambda_i)^t \bw_{0,i}$. Let $\mathcal{K}=\{\boldsymbol{x}:\boldsymbol{x}_{d-m_2+1}=\cdots=\boldsymbol{x}_{d}=0\} \cup \{\boldsymbol{x}:\boldsymbol{x}_{1}=\cdots=\boldsymbol{x}_{m_1}=0\}$. Obviously, $\mathcal{K}$ is a zero-measure set. Then, we have that $\lim_{t\rightarrow \infty }\frac{\bw_t}{ \Vert \bw_t \Vert }\in \mathcal{A}$ if and only if $\vert 1-\eta\lambda_{1}\vert >\vert 1-\eta\lambda_{d}\vert$, which gives $\eta>\frac{2}{\lambda_1+\lambda_d}$.
\end{proof}

\begin{proof}[Proof of Proposition \ref{prop: convergence_direction_momentum}] Let $\lambda_1,\cdots,\lambda_d$, $ \mathcal{K}$, $m_1$, and $m_2$ be defined in the proof of Proposition \ref{prop: convergence_direction}. Then, the update rule of GDM gives
\begin{align*}
     \bw_{t+1,i}= &\bw_{t,i}+\mu(\bw_{t,i}-\bw_{t-1,i})-(1-\mu)\eta_{\ef} \partial_i f(\bw_{t,i})
     \\
     =&\bw_{t,i}+\mu(\bw_{t,i}-\bw_{t-1,i})-(1-\mu)\eta_{\ef} \lambda_i\bw_{t,i}.
\end{align*}
Solving the above series gives 
\begin{equation*}
    \bw_{t,i} = c_{i,1} d_{i,1}^t+c_{i,2} d_{i,2}^t,
\end{equation*}
where $d_{i,1}=\frac{(1+\mu)-(1-\mu)\eta_{\ef}\lambda_i}{2}+\sqrt{\left(\frac{(1+\mu)-(1-\mu)\eta_{\ef}\lambda_i}{2}\right)^2-\mu}$, and $d_{i,2}=\frac{(1+\mu)-(1-\mu)\eta_{\ef}\lambda_i}{2}-\sqrt{\left(\frac{(1+\mu)-(1-\mu)\eta_{\ef}\lambda_i}{2}\right)^2-\mu}$.

Therefore, $\lim_{t\rightarrow \infty} \frac{\bw_t}{ \Vert \bw_t \Vert  }\in \mathcal{A}$ if and only if $\max\{d_{1,1},d_{1,2}\} >\max_{i\ne 1}\{d_{i,1},d_{i,2}\} $. On the other hand, note that $g(x)= \max\{\vert \frac{(1+\mu)-x}{2}+\sqrt{\left(\frac{(1+\mu)-x}{2}\right)^2-\mu}\vert , \vert \frac{(1+\mu)-x}{2}-\sqrt{\left(\frac{(1+\mu)-x}{2}\right)^2-\mu}\vert \}$ is symmetric with respect to $x=1+\mu$, and the maximum value of $g(x)$ over any interval $[a,b]$ is achieved at $a$ or $b$, then  $\lim_{t\rightarrow \infty} \frac{\bw_t}{ \Vert \bw_t \Vert  }\in \mathcal{A}$ if and only if $(1-\mu)\eta_{\ef}\lambda_1+(1-\mu)\eta_{\ef}\lambda_d>2(1+\mu)$ and $(1-\mu)\eta_{\ef}\lambda_1> (1+\sqrt{\mu})^2$.

The proof is completed.
\end{proof}

\subsection{Proof of Proposition \ref{prop:gdgdm}}

Without loss of generality, choose $\bom_0=\frac{\nabla f(\bw_1^{\GDM})}{1-\mu}$ (since the influence of $\bom_0$ diminishes exponentially fast). To begin with, define an auxiliary sequence as $\bu_1=\bw_1^{\GDM}-\frac{\mu}{1-\mu}\eta_{\ef} \nabla f(\bw_1^{\GDM})$ and $\bu_t=\frac{\bw_{t}^{\GDM}-\mu \bw_{t-1}^{\GDM}}{1-\mu}$. One can easily verify that the update rule of GDM is equivalent to
 \begin{gather}
\label{lem: auxilliary}
    \bu_{t+1}=\bu_t - \eta_{\ef} \nabla  f
(\bw_t^{\GDM}),
    \bw_{t+1}= (1-\mu) \bu_{t+1} + \mu \bw_t^{\GDM}.
\end{gather}
When $t=1$, we have
\begin{equation*}
    \Vert \nabla f(\bw^{\GDM}) \Vert^2 = \Vert \nabla f(\bw^{\GD}) \Vert^2 
\end{equation*}
by definition.
We then show that when $ t\ge k\ge 2$, $f(\bu_k)-f(\bu_{k-1})\approx f(\bw_k^{\GD})-f(\bw_{k-1}^{\GD})$ and $\Vert \nabla f(\bw_k^{\GD}) \Vert \approx \Vert \nabla f(\bw_k^{\GDM}) \Vert$ by induction. Suppose that the claim holds for the $k$-th iteration. Then, for the $(k+1)$-th iteration, by Taylor's expansion, we have
\small
\begin{gather*}
    f(\bu_k)\approx f(\bw_{k-1}^{\GDM})+\langle \bu_k-\bw_{k-1}^{\GDM},\nabla f(\bw_{k-1}^{\GDM}) \rangle+\frac{H(\bw_{k-1}^{\GDM},\bu_k-\bw_{k-1}^{\GDM})}{2} \Vert  \bu_k-\bw_{k-1}^{\GDM} \Vert^2 ,
    \\
        f(\bu_{k+1})\approx f(\bw_k^{\GDM})+\langle \bu_{k+1}-\bw_k^{\GDM},\nabla f(\bw_k^{\GDM}) \rangle+\frac{H(\bw_k^{\GDM},\bu_{k+1}-\bw_k^{\GDM})}{2} \Vert  \bu_{k+1}-\bw_k^{\GDM} \Vert^2 ,
        \\
          f(\bw_k^{\GDM})\approx f(\bw_{k-1}^{\GDM})+\langle \bw_k^{\GDM}-\bw_{k-1}^{\GDM},\nabla f(\bw_{k-1}^{\GDM}) \rangle+\frac{H(\bw_{k-1}^{\GDM},\bw_k^{\GDM}-\bw_{k-1}^{\GDM})}{2} \Vert  \bw_k^{\GDM}-\bw_{k-1}^{\GDM} \Vert^2.
\end{gather*}
\normalsize
By the definition of $\bu_k$, we have that $\bu_k-\bw_{k-1}^{\GDM}=\frac{\bw_k^{\GDM}-\bw_{k-1}^{\GDM}}{1-\mu}$, and thus $H(\bw_{k-1}^{\GDM},\bu_k-\bw_{k-1}^{\GDM}) = H(\bw_{k-1}^{\GDM},\bw_k^{\GDM}-\bw_{k-1}^{\GDM}) \approx 0$. Similarly, we have $H(\bw_k^{\GDM},\bu_{k+1}-\bw_k^{\GDM})\approx 0$ and $H(\bw_{k-1}^{\GDM},\bw_k^{\GDM}-\bw_{k-1}^{\GDM})\approx 0$. Therefore, summing up the above three equations, we have
\begin{equation*}
    f(\bu_{k+1})\approx f(\bu_k) +\langle \bu_{k+1}-\bw_k^{\GDM},\nabla f(\bw_k^{\GDM}) \rangle + \langle \bw_k^{\GDM}-\bu_k,\nabla f(\bw_{k-1}^{\GDM}) \rangle.
\end{equation*}

Since 
\begin{align*}
    &\langle \bw_k^{\GDM}-\bw_{k-1}^{\GDM},\nabla f(\bw_{k-1}^{\GDM}) \rangle
    \\
    =& \langle \bw_k^{\GDM}-\bw_{k-1}^{\GDM},\nabla f(\bw_k^{\GDM}) \rangle -\langle \bw_k^{\GDM}-\bw_{k-1}^{\GDM},\nabla f(\bw_k^{\GDM})-\nabla f(\bw_{k-1}^{\GDM}) \rangle
    \\
    \approx &\langle \bw_k^{\GDM}-\bw_{k-1}^{\GDM},\nabla f(\bw_k^{\GDM}) \rangle-H(\bw_{k-1}^{\GDM},\bw_k^{\GDM}-\bw_{k-1}^{\GDM})   \Vert  \bw_k^{\GDM}-\bw_{k-1}^{\GDM} \Vert^2
    \\
    \approx &\langle \bw_k^{\GDM}-\bw_{k-1}^{\GDM},\nabla f(\bw_k^{\GDM}) \rangle,
\end{align*}
we further have 
 \begin{align*}
    f(\bu_{k+1})\approx& f(\bu_k) +\langle \bu_{k+1}-\bw_k^{\GDM},\nabla f(\bw_k^{\GDM}) \rangle + \langle \bw_k^{\GDM}-\bu_k,\nabla f(\bw_k^{\GDM}) \rangle
    \\
    =& f(\bu_k) +\langle \bu_{k+1}-\bu_k^{\GDM},\nabla f(\bw_k^{\GDM}) \rangle 
    \\
    =& f(\bu_k) -\eta_{\ef}\Vert \nabla f(\bw_k^{\GDM}) \Vert^2 .
\end{align*}

Following the same routine, we obtain
 \begin{align*}
    f(\bw_{k+1}^{\GD})\approx & f(\bw_k^{\GD}) -\eta_{\ef}\Vert \nabla f(\bw_k^{\GD}) \Vert^2 ,
\end{align*}
and thus we obtain  $f(\bu_{k+1})-f(\bu_k)\approx f(\bw_{k+1}^{\GD})-f(\bw_k^{\GD})$ due to  that $\Vert \nabla f(\bw_k^{\GD}) \Vert^2\approx \Vert \nabla f(\bw_k^{\GDM}) \Vert^2$ by the induction hypothesis. 

Meanwhile, we have
\begin{align*}
    \Vert \nabla f(\bw_{k+1}^{\GD}) \Vert^2\approx&\Vert \nabla f(\bw_k^{\GD}) \Vert^2+\langle \nabla f(\bw_k^{\GD}), \nabla f(\bw_{k+1}^{\GD})-\nabla f(\bw_k^{\GD})\rangle
    \\
    \approx &\Vert \nabla f(\bw_k^{\GD}) \Vert^2+\langle \nabla f(\bw_k^{\GD}), \nabla^2 f(\bw_k^{\GD}) (\nabla f(\bw_{k+1}^{\GD})-\nabla f(\bw_k^{\GD}))\rangle
    \\
    =&\Vert \nabla f(\bw_k^{\GD}) \Vert^2+\eta_{\ef}H(\bw_k^{\GD}, \bw_{k+1}^{\GD}-\bw_k^{\GD})\Vert \nabla f(\bw_{k+1}^{\GD}) \Vert^2\approx\Vert \nabla f(\bw_k^{\GD}) \Vert^2.
\end{align*}
Following the similar routine, we obtain
\begin{equation*}
     \Vert \nabla f(\bw_{k+1}^{\GDM}) \Vert^2 \approx \Vert \nabla f(\bw_k^{\GDM}) \Vert^2,
\end{equation*}
and thus we obtain $\Vert \nabla f(\bw_{k+1}^{\GD}) \Vert^2\approx \Vert \nabla f(\bw_{k+1}^{\GDM}) \Vert^2$  due to  that $\Vert \nabla f(\bw_k^{\GD}) \Vert^2\approx \Vert \nabla f(\bw_k^{\GDM}) \Vert^2$ by the induction hypothesis. 

As a conclusion, we obtain that $f(\bu_t)-f(\bu_1)\approx f(\bw_t^{\GD})- f(\bw_1^{\GD})$.

Meanwhile, as discussed above, we have
\begin{gather*}
    f(\bu_t)\approx f(\bw_{k-1}^{\GDM})+\langle \bu_t-\bw_{t-1}^{\GDM},\nabla f(\bw_{t-1}^{\GDM}) \rangle ,
        \\
          f(\bw_{t}^{\GDM})\approx f(\bw_{t-1}^{\GDM})+\langle \bw_{t}^{\GDM}-\bw_{t-1}^{\GDM},\nabla f(\bw_{t-1}^{\GDM}) \rangle.
\end{gather*}
Summing up the two equations gives 
\begin{align*}
    f(\bu_t) \approx& f(\bw_t^{\GDM})+\langle \bu_t-\bw_{t}^{\GDM},\nabla f(\bw_{t-1}^{\GDM}) \rangle
    \\
    =& f(\bw_t^{\GDM})+\frac{\mu}{1-\mu}\langle \bw^{\GDM}_t-\bw^{\GDM}_{t-1},\nabla f(\bw_{t-1}^{\GDM}) \rangle
    \\
    =& f(\bw_t^{\GDM})-\frac{\mu}{1-\mu}\eta_{\ef} \left\langle (1-\mu)\sum_{s=1}^{t-1} \mu^{t-1-s} \nabla f(\bw_{s}^{\GDM})+\mu^{t-1} \nabla f(\bw_{1}^{\GDM}),\nabla f(\bw_{t-1}^{\GDM}) \right\rangle.
\end{align*}
Since
\small
\begin{align*}
    &-\eta_{\ef}\left\langle (1-\mu)\sum_{s=1}^{t-1} \mu^{t-1-s} \nabla f(\bw_{s}^{\GDM})+\mu^{t-1} \nabla f(\bw_{1}^{\GDM}),\nabla f(\bw_{t-1}^{\GDM}) \right\rangle
    \\
    =&-(1-\mu) \eta_{\ef}\Vert \nabla f(\bw_{t-1}^{\GDM})\Vert^2+\mu\left\langle \bw_{t-1}^{\GDM}-\bw_{t-2}^{\GDM},\nabla f(\bw_{t-1}^{\GDM}) \right\rangle
    \\
    =&-(1-\mu) \eta_{\ef}\Vert \nabla f(\bw_{t-1}^{\GDM})\Vert^2+\mu\left\langle \bw_{t-1}^{\GDM}-\bw_{t-2}^{\GDM},\nabla f(\bw_{t-2}^{\GDM}) \right\rangle+\mu\left\langle \bw_{t-1}^{\GDM}-\bw_{t-2}^{\GDM},-\nabla f(\bw_{t-1}^{\GDM})+\nabla f(\bw_{t-2}^{\GDM}) \right\rangle
    \\
    \approx & -(1-\mu) \eta_{\ef}\Vert \nabla f(\bw_{t-1}^{\GDM})\Vert^2+\mu\left\langle \bw_{t-1}^{\GDM}-\bw_{t-2}^{\GDM},\nabla f(\bw_{t-2}^{\GDM}) \right\rangle+\mu H(\bw_{t-2}^{\GDM},\bw_{t-1}^{\GDM}-\bw_{t-2}^{\GDM})\left\Vert \bw_{t-1}^{\GDM}-\bw_{t-2}^{\GDM}\right\Vert^2
    \\
    \approx& -(1-\mu) \eta_{\ef}\Vert \nabla f(\bw_{t-1}^{\GDM})\Vert^2+\mu\left\langle \bw_{t-1}^{\GDM}-\bw_{t-2}^{\GDM},\nabla f(\bw_{t-2}^{\GDM}) \right\rangle
    \\
    \approx & \cdots
    \\
    \approx &-\eta_{\ef} (1-\mu)\sum_{s=1}^{t-1} \mu^{t-1-s} \Vert \nabla f(\bw_{s}^{\GDM}) \Vert^2-\eta_{\ef}\mu^{t-1} \Vert \nabla f(\bw_{1}^{\GDM})\Vert^2
    \\
    \approx& -\eta_{\ef}\Vert \nabla f(\bw_{1}^{\GDM})\Vert^2,
\end{align*}
\normalsize
and 
\begin{align*}
    f(\bu_1) \approx &f(\bw_1^{\GDM})+\langle \nabla f (\bw_1^{\GDM}),\bu_1-\bw_1^{\GDM} \rangle+\frac{H(\bw_1^{\GDM},\bu_1-\bw_1^{\GDM})}{2} \Vert \bu_1-\bw_1^{\GDM} \Vert^2
    \\
    =&f(\bw_1^{\GDM})+\langle \nabla f (\bw_1^{\GDM}),\bu_1-\bw_1^{\GDM} \rangle+\frac{H(\bw_1^{\GDM},\bw_2^{\GDM}-\bw_1^{\GDM})}{2} \Vert \bu_1-\bw_1^{\GDM} \Vert^2
    \\
    \approx& f(\bw_1^{\GDM})+\langle \nabla f (\bw_1^{\GDM}),\bu_1-\bw_1^{\GDM} \rangle
    \\
    =& f(\bw_1^{\GD})-\eta_{\ef}\Vert \nabla f(\bw_{1}^{\GDM})\Vert^2.
\end{align*}
As a conclusion, we have
\begin{equation*}
    f(\bw_t^{\GDM})\approx   f(\bu_t) +\eta_{\ef}\frac{\mu}{1-\mu}\Vert \nabla f(\bw_{1}^{\GDM})\Vert^2 \approx  f(\bw_t^{\GD})- f(\bw_1^{\GD})+f(\bu_1)+\eta_{\ef}\Vert \nabla f(\bw_{1}^{\GDM})\Vert^2\approx f(\bw_t^{\GD}).
\end{equation*}

The proof is completed.

\end{document}